\documentclass[letterpaper, 10 pt, conference]{ieeeconf}

\IEEEoverridecommandlockouts                              

\usepackage[caption=false,font=footnotesize]{subfig}
\usepackage {graphicx}

\usepackage{cite}
\usepackage[utf8]{inputenc}
\usepackage{amsmath}
\usepackage[ruled, vlined, linesnumbered]{algorithm2e}
\usepackage{amssymb}
\usepackage{amsfonts}
\usepackage[bb=boondox]{mathalfa}
\usepackage[utf8]{inputenc}
\usepackage[english]{babel}
\usepackage{bm}
\usepackage{accents}
\newtheorem{theorem}{\bf Theorem}

\newtheorem{definition}{\bf Definition}

\newtheorem{remark}{\bf Remark}

\newtheorem{lemma}{\bf Lemma}
\newtheorem{claim}{\bf Claim}

\usepackage[caption=false,font=footnotesize]{subfig}

\usepackage[]{xcolor}
\usepackage{physics}

\usepackage{hyperref}
\hypersetup{
    colorlinks,
    linkcolor={black!50!black},
    citecolor={black!50!black},
}

\title{\LARGE \bf
A Smoothing Algorithm for Minimum Sensing Path Plans in Gaussian Belief Space  
}

\author{ Ali Reza Pedram$^1$, and Takashi Tanaka$^{2}$ 
\thanks{$^{1}$ Walker Department of Mechanical Engineering, University of Texas at Austin. {\tt\small apedram@utexas.edu}.
         $^{2}$Department of Aerospace Engineering and Engineering Mechanics, University of Texas at Austin.  {\tt\small ttanaka@utexas.edu}. }%
}

\begin{document}

\maketitle
\thispagestyle{empty}
\pagestyle{empty}

\begin{abstract}
This paper explores minimum sensing navigation of robots in environments cluttered with
obstacles. The general objective is to find a path plan to a goal region that requires minimal
sensing effort. In [1], the information-geometric RRT* (IG-RRT*) algorithm was  proposed to
efficiently   find   such   a   path.   However,   like   any   stochastic   sampling-based   planner,   the
computational complexity of IG-RRT* grows quickly, impeding its use with a large number of
nodes. To remedy this limitation, we suggest running IG-RRT* with a moderate number of
nodes, and then using a smoothing algorithm to adjust the path obtained. To develop a
smoothing algorithm, we explicitly formulate the minimum sensing path planning problem as
an   optimization   problem.   For   this   formulation,   we   introduce   a   new   safety   constraint   to
impose a bound on the probability of collision with obstacles in continuous-time, in contrast
to   the   common   discrete-time   approach.   The   problem   is   amenable   to   solution   via   the
convex-concave   procedure   (CCP).   We   develop   a   CCP   algorithm   for   the   formulated
optimization and use this algorithm for path smoothing. We demonstrate the efficacy of the
proposed approach through numerical simulations.
\end{abstract}

\section{Introduction}
\label{sec:Intro}
Advancements in sensing and computer vision techniques over the last decades have facilitated the acquisition of  ample amounts of information for robot navigation. However, using all the available data can result in  long processing times; draining the available computational power and communication bandwidth. 
One popular approach for managing the overhead of intense information processing is the strategic use of available sensory data as opposed to full deployment. For instance, in \cite{carlone2018attention,zhao2020good}, several feature selection algorithms are proposed and compared for efficient visual odometry/visual simultaneous localization and mapping. In \cite{carlone2018attention,zhao2020good} the features that contribute the most to state estimation are incorporated meanwhile others are ignored. In the same line of research, \cite{tzoumas2020lqg} incorporated a restriction on the sensing budget into an optimal control problem and proposed an algorithm for control and sensing co-design. In \cite{pedram2021dynamic}, an attention mechanism is proposed for the optimal allocation of sensing resources.

Executing different path plans requires different amounts of sensory data.
Thus, it is crucial for an autonomous system to be able to find path plans that require minimal sensing. To achieve this capability, \cite{pedram2022gaussian}  and \cite{pedram2021rationally} established the minimum sensing path planning paradigm, 
where a pseudo-metric is introduced for Gaussian belief space to quantify the
augmented  control and information costs incurred in the transition between two arbitrary states. The work \cite{pedram2022gaussian} proposed  an asymptotically optimal \cite{solovey2020revisiting} sampling-based motion planner, referred to as information-geometric RRT* (IG-RRT*) to find the shortest path in the introduced metric.

The success of IG-RRT* in providing optimal path plans in obstacle-cluttered environments is verified through simulation in \cite{pedram2022gaussian}. However, like any RRT* algorithm, the time complexity of IG-RRT*  with $n$ nodes grows as $\mathcal{O}(n\log n)$ \cite{karaman2011sampling}. This complexity restricts finding a precise approximation of the optimal path plan directly by deploying  IG-RRT* for a large number of steps. We thus propose a two-stage procedure to circumvent this computational limitation of IG-RRT*. In the first stage, the algorithm is run for a moderate number of iterations to find an approximation of the optimal path. The path sought in this stage is ``jagged" in most cases due to the stochastic nature of IG-RRT*. In the second stage, this path is ``smoothed" toward the optimal path using gradient-based solvers.

Performing the second stage requires the minimum sensing path problem to be formulated explicitly as a shortest path problem with safety constraints ensuring a small probability of collision with obstacles. The explicit formulation of the shortest path problem in the belief space with safety constraints is a well-studied problem, and a core element of chance-constrained (CC) motion planners \cite{blackmore2011chance}. 
CC planners impose a set of deterministic constraints (e.g., see \cite[Equation (18)]{blackmore2011chance}) which ensure the instantaneous probability of collision with half-space obstacles is bounded at discrete time steps. 
The probability of collision with polyhedral obstacles is bounded 
using probabilistic bounds like  Boole's inequality \cite{ono2015chance}. Constraints that aim to bound the collision probability in continuous time, that is, in the transition between discrete-time steps, are less explored. \cite{ariu2017chance} and \cite{oguri2019convex} are among a few papers that studied continuous-time safety constraints using the reflection principle and cumulative Lyapunov exponent formulation, respectively. Yet, these papers only provide constraints to bound the probability of collision with half-space obstacles.

In this work, we derive novel safety constraints that directly bound the probability of collision with polyhedral obstacles in both discrete and continuous time. Our derivation is obtained using  strong duality and the theorem of alternatives \cite{boyd2004convex}. Using the proposed safety constraints, we formulate the minimum sensing path problem explicitly. We partially convexify the resulting optimization and represent it as a difference of convex (DOC) program. By exploiting the DOC form, we devise an iterative algorithm that starts from the feasible solution obtained by IG-RRT* at stage one, and monotonically 
smooths the path toward a locally optimal path using gradient-based solvers. The devised algorithm is the implementation of the convex-concave procedure (CCP) \cite{yuille2003concave} for the formulated problem.

\emph{Notation}:  We write vectors in lower case $x$ and matrices in upper case $X$. Let $\mathbb{S}_{+}^d=\left\{P\in\mathbb{R}^{d \times d}: P=P^\top \succ 0 \right\}$ and $\| x \|$ denote the 2-norm of $x$. For positive integers $i<j$, $[i:j]$ denotes the set $\{i,\; i+1,\; \dots,\; j\}$ and  $[j]=[1:j]$. $\bm{x} \sim \mathcal{N}(x,P)$ denotes a Gaussian random variable with mean $x$ and covariance $P$. The unit interval $\{a \in \mathbb{R}: 0\leq a \leq 1\}$ is denoted by $int$ and the zero vector in $\mathbb{R}^f$ by $0_f$. For vectors $\geq$ denotes element-wise inequality, e.g. $a\geq 0_f$.
\vspace{-0.2cm}
\section{Preliminaries}
\label{sec:Prelim}
As in \cite{pedram2021rationally} and \cite{pedram2022gaussian}, we model the minimum sensing navigation of a robot as a shortest path problem in Gaussian belief space $\mathbb{R}^d\times \mathbb{S}^d_{+}$. In this setting, we search for the optimal joint control-sensing plan that drives the robot from an initial state to a target region while achieving the minimum steering cost (defined in Subsection~\ref{sec:cost}) and a bounded probability of collision with obstacles.  
\vspace{-0.2cm}
\subsection{Assumed Dynamics}
We assume that the robot's dynamics are governed by a controlled Ito process
\vspace{-0.2cm}
\begin{equation}
\label{eq:dyn_cont}
    d\bm{x}(t)=\bm{v}(t) dt+ N^{\frac{1}{2}} d\bm{b}(t), 
\end{equation}
where $\bm{x}(t) \in \mathbb{R}^d$ is the state  of the robot at time $t$ with initial $\bm{x}(0)\sim\mathcal{N}(x_0, P_0)$, $\bm{v}(t)$ is the velocity input command, $\bm{b}(t)$ is  $d$-dimensional standard
Brownian motion, and $N \in \mathbb{S}^d$ is the process noise intensity.\footnote{ In practice, the path planning strategy we propose is applicable even if the robot's actual dynamics are different from  
\eqref{eq:dyn_cont}. See \cite[Section~II.A]{pedram2022gaussian} for further discussion.} We assume that the robot is commanded at constant periods of $\Delta t$, and the control input is applied using a zero-order hold converter. Under these assumptions, \eqref{eq:dyn_cont} can be discretized via
\vspace{-0.2cm}
\begin{equation}
\label{eq:dyn}
\begin{split}
    &\bm{x}_{k} = \bm{x}_{k-1}+ \bm{u}_{k-1}+\bm{w}_{k-1},
\end{split}
\end{equation}
where $\bm{x}_k$ is a state of the robot at time $t=k\Delta t $, $\bm{u}_{k-1}=\bm{v}(t_{k-1})\Delta t$, and $\bm{w}_{k-1}\sim\mathcal{N}(0, W := N \Delta t)$.

Let the probability distributions of the robot's state at
a given time step $k-1$ be parameterized by a Gaussian model $\bm{x}_{k-1} \sim \mathcal{N}(x_{k-1}, P_{k-1})$. After exerting the deterministic feed-forward control input $u_{k-1}$, the covariance propagates linearly to $\hat{P}_{k}=P_{k-1}+W$, and the 
robot's state becomes  $(x_{k}=x_{k-1}+u_{k-1}, \hat{P}_{k})$ prior to the measurement at $t_{k}$. After the measurement at $t_k$, the covariance is reduced to the posterior covariance $P_{k}\triangleq (\hat{P}_k^{-1}+S_k)^{-1}\preceq \hat{P}_{k}$, where $S_k\in \mathbb{S}_{+}^d$ is the information content of the measurement.\footnote{ See \cite[Section~III.E]{pedram2022gaussian} 
that explains how a given $S_k$ is realized by an appropriate choice of sensors.}


\subsection{Steering Cost}
\label{sec:cost}
The cost incurred during $t_{k-1}\rightarrow t_{k}$ is defined as the weighted sum of control cost $D_{\text{cont},k}$ and the information acquisition cost $D_{\text{info},k}$, as 
\[
\mathcal{D}_k\triangleq \mathcal{D}_{\text{cont},k}+ \alpha \mathcal{D}_{\text{info},k}, 
\]
where $\alpha$ is the weight factor. The control cost is simply the control input power $\mathcal{D}_{\text{cont},k} \triangleq \|u_{k-1}\|^2=\|x_{k}-x_{k-1}\|^2$.
The information cost is the entropy reduction incurred at the end of transition $t_{k-1}\rightarrow t_k$, defined as
$\mathcal{D}_{\text{info},k} \triangleq \frac{1}{2}\log\det \hat{P}_{k} - \frac{1}{2}\log\det P_{k}$.

\subsection{Collision Constraints}

We assume the path planning is to be performed inside a polyhedral domain $\mathcal{X}_{\text{dom}} := \{ x\in \mathbb{R}^d: a_{\text{dom},l}^\top x \leq b_{\text{dom},l}, l \in [L] \}$, which is filled with polyhedral obstacles
$\mathcal{X}_{\text{obs}}^m := \{ x\in \mathbb{R}^d: A_{\text{obs},m} \; x \leq b_{\text{obs},m}\}$  for $m\in [M]$. We assume the target region is also defined as 
polyhedral region $\mathcal{X}_{\text{tar}} := \{ x\in \mathbb{R}^d: a_{\text{tar},n}^\top x \leq b_{\text{tar},n}, n \in [N]\}$. Fig.~\ref{fig:sample_enviroment} shows a sample environment with $L=4$, $M=2$, and $N=4$.  

\begin{figure}[h]
    \centering
    \includegraphics[trim = 0.0cm 1cm 0cm 0.7cm, clip=true, width=0.9\columnwidth]{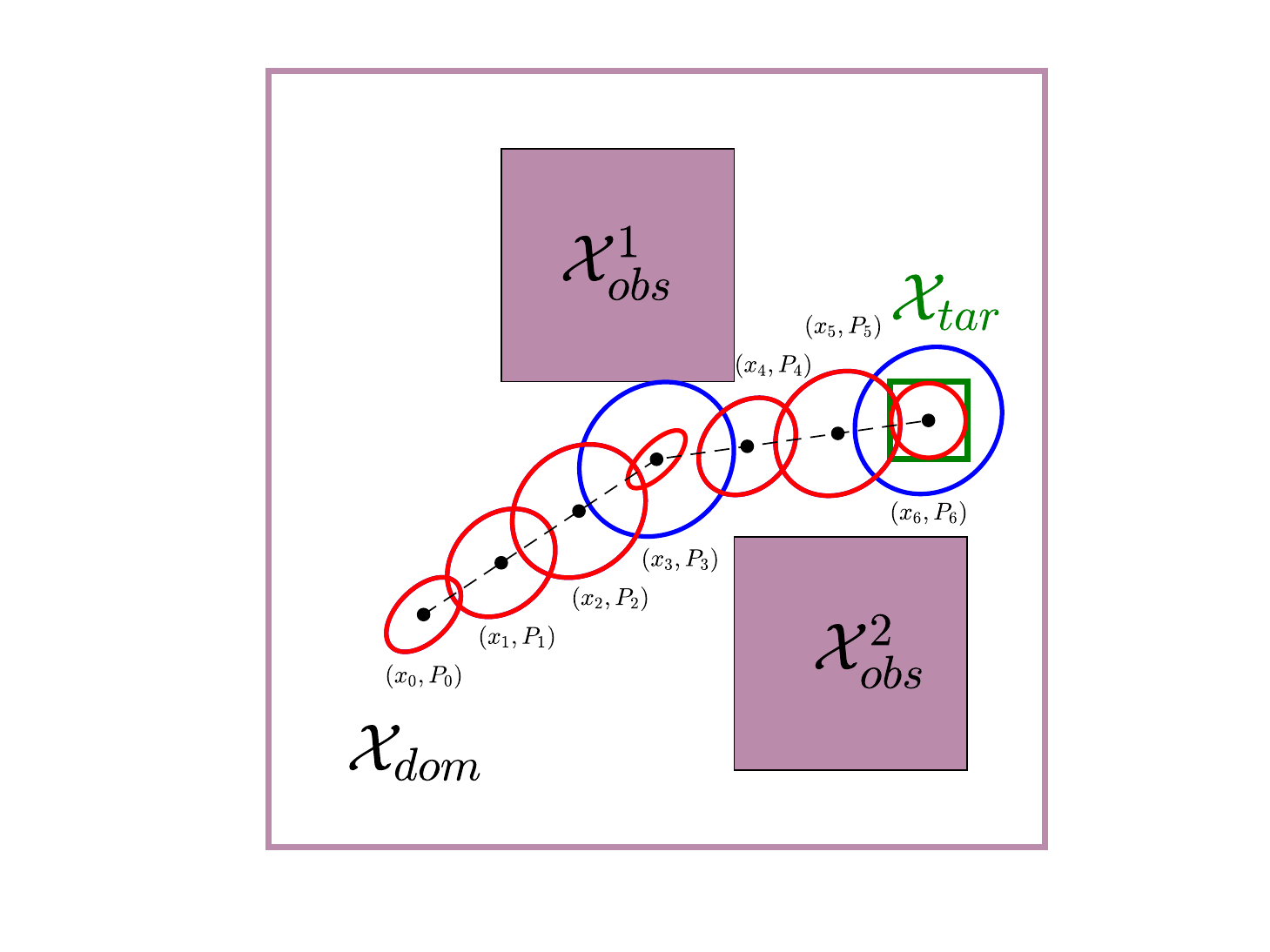}
    \caption{A sample path planning environment and a sample belief path with $6$ steps. The blue ellipses show the  $90\%$ prior confidence ellipses $\mathcal{E}_{\chi}(x_k,\hat{P}_k)$, whereas the red ones show the $90\%$ posterior confidence $\mathcal{E}_{\chi}(x_k,{P}_k)$. At time steps when no measurement is made, we have $P=\hat{P}$, and thus the prior and posterior ellipses are identical.}
    \label{fig:sample_enviroment}
\end{figure}
\vspace{-0.3cm}
We define the $\textup{Pr}\%$ confidence ellipse for a Gaussian distribution $\bm{x}\sim \mathcal{N}(x_i,P_i)$ as $\mathcal{E}_{\chi}(x_i,P_i) := \{{x}\in \mathbb{R}^d: ({x}-x_i)^\top P^{-1}({x}-x_i) \leq \chi^2(\textup{Pr})\}$, where $\chi^2({\textup{Pr}})$ is the $\textup{Pr}$-th quantile of the Chi-squared distribution. For simplicity, we use $\chi^2$ in the sequel. Fig.~\ref{fig:sample_enviroment} depicts a sample belief path with $6$ steps using confidence ellipses, where two measurements are performed at time-steps $k=3$ and $6$ with no measurement at all other time-steps $k$ (which means $S_k=0$ for $k\notin \{1,3\}$).
\begin{definition}
For a fixed confidence level $\chi^2>0$,
we say a belief state $(x_i,P_i)$ is collision-free if the confidence ellipse $\mathcal{E}_{\chi}(x_i,P_i)$ has empty overlap with obstacles $\mathcal{X}_{\text{obs}}^m$ for all $m\in [M]$ and is contained in $\mathcal{X}_{\text{dom}}$. The collision-free constraint can be written as 
\vspace{-0.2cm}
\begin{align}
\label{eq:safe_a}
&x_{\text{obs}} \notin \mathcal{E}_{\chi}(x_i, P_i), \;\ \forall x_{\text{obs}}\in \mathcal{X}^m_{\text{obs}}, \; m\in M,\\
\label{eq:safe_b}
&\mathcal{E}_{\chi}(x_i,P_i) \subseteq \mathcal{X}_{\text{dom}}.
\end{align}
\end{definition}

At first glance, \eqref{eq:safe_a} and \eqref{eq:safe_b} seem to have two seperate mathematical forms. However, \eqref{eq:safe_b} can be written as
\begin{align*}
    x_{\text{out}} \notin \mathcal{E}_{\chi}(x_i,P_i), \; \forall x_{\text{out}}\in \mathcal{X}_{\text{out}},
\end{align*}
where $\mathcal{X}_{\text{out}} := \mathbb{R}^d \backslash \mathcal{X}_{\text{dom}}$. We can rewrite $ \mathcal{X}_{\text{out}}$ as the union of half-space ($1$-faced polyhedral) obstacles $\mathcal{X}_{\text{out}}^l=\{x\in \mathbb{R}^d: -a_{\text{dom},l}^\top x \leq -b_{\text{dom},l} \}$ for $l \in [L]$. Hence, \eqref{eq:safe_a} and \eqref{eq:safe_b} can be jointly written as
$x_{\text{o}} \notin \mathcal{E}_{\chi}(x_i, P_i), \;\ \forall x_{\text{o}}\in \mathcal{X}^j_{\text{O}}, \; j\in [J:=M+L]$,
where $\mathcal{X}^j_{\text{O}} \in \{\mathcal{X}^1_{\text{obs}}, \dots, \mathcal{X}^M_{\text{obs}}, \mathcal{X}^1_{\text{out}}, \dots, \mathcal{X}^L_{\text{out}}\}$.



\begin{definition}
For a fixed confidence level  $\chi^2>0$, we say belief state $(x_i,P_i)$ is an admissible final state if $\mathcal{E}_{\chi}(x_i,P_i)$ is contained in $\mathcal{X}_{\text{tar}}$ (i.e., $\mathcal{E}_{\chi}(x_i,P_i) \subseteq \mathcal{X}_{\text{tar}}$). For instance, $\mathcal{E}_{\chi}(x_6,P_6)$ in Fig.~\ref{fig:sample_enviroment} is an admissible final belief state.
\end{definition}

Similar to $\mathcal{X}_{\text{out}}$,  region $\mathcal{X}_{\text{tout}}:= \mathbb{R}^d\backslash\mathcal{X}_{\text{tar}}$ can be thought as the union of half-space obstacles  $\mathcal{X}_{\text{tout}}^n=\{x\in \mathbb{R}^d: -a_{n}^\top x \leq -b_{n} \}$ for $n\in [N]$. Thus,  $\mathcal{E}_{\chi}(x_i,P_i)$ is an admissible final state iff $x_{\text{tout}} \notin \mathcal{E}_{\chi}(x_{i},P_i), \forall x_{\text{tout}}\in \mathcal{X}^n_{\text{tout}}, \; \forall n\in [N]$.

Governed by dynamics \eqref{eq:dyn_cont}, it is easy to verify that the state of the robot in transition $t_{k-1}\rightarrow t_k$ can be parameterized by $s \in int$ as $(x_k[s]= x_{k-1}+ s \Delta x_{k} ,P[s] = P_{k-1}+ s W )$, where $\Delta x_{k}:= x_k-x_{k-1}$.

\begin{definition}
\label{def:three}
For a fixed confidence level $\chi^2>0$, we say that the transition from $x_{k-1}$ to $x_{k}$ with initial covariance $P_{k-1}$ is collision-free if for all $s \in int$, $\mathcal{E}_{\chi}(x[s],P[s])$ has empty overlap with obstacles $\mathcal{X}_{\text{obs}}^m$ for all $m\in[M]$, and $\mathcal{X}_{\text{out}}$. Mathematically, it is equivalent to $x_{\text{o}} \notin \mathcal{E}_{\chi}(x_k[s], P_k[s]), \; \forall x_{\text{o}}\in \mathcal{X}^j_{\text{O}}, \; j\in [J],\; \forall s\in int$.
\end{definition}

\vspace{-0.2cm}
\section{Problem Formulation}
\label{sec:Formulation}
Let's fix the number of time steps $K$ and confidence level $\chi^2$. Introducing information matrix $Q_{k}:=({P}_{k-1}+W)^{-1}= P_k^{-1}-S_k$, the shortest path problem with respect to the proposed steering cost can be formulated as
\vspace{-0.2cm}
\begin{subequations}
\label{eq:main_info}
\begin{align}
\nonumber
    \min \;\; & \sum_{k=1}^{K} \big(\|x_{k}-x_{k-1}\|^2+\frac{\alpha}{2} \log\det ({Q}_{k}+S_k)\!
    \\ \label{eq:main_a_info} 
    & \quad \quad  -\!\frac{\alpha}{2} \log\det Q_{k} \big) \\
    \label{eq:main_lossless_info}
    \text{s.t.} \;\; &Q_{k}^{-1} = (Q_{k-1}+S_{k-1})^{-1}+W, \;  \forall k \in [K],\\
    \nonumber 
    & x_{\text{o}} \notin \mathcal{E}_{\chi}(x_{k}[s],(Q_k+S_k)^{-1}[s]),\\ \label{eq:main_collision_info}
    & \forall k\in [K],\; \forall x_{\text{o}}\in \mathcal{X}^j_{\text{O}}, \; \forall j\in [J],\; \forall s \in int\\
    \nonumber
     &x_{\text{tout}} \notin \mathcal{E}_{\chi}(x_{K},(Q_K+S_K)^{-1}),\\ \label{eq:main_final_info}
     &\forall x_{\text{tout}}\in \mathcal{X}^n_{\text{tout}}, \; \forall n\in [N],
\end{align}
\end{subequations} where the minimization is performed over $\{x_k, {Q}_k, S_k\}_{k=1}^K$, and  $x_0$ and $Q_0+S_0 := P_{\text{0}}^{-1}$ are given.
Constraint \eqref{eq:main_lossless_info} is the Kalman filter iteration, \eqref{eq:main_collision_info} states all transitions $t_{k-1}\rightarrow t_{k}$ are safe (collision-free), and  \eqref{eq:main_final_info} ensures that final belief is an admissible final state.  
 We can define a relaxation of problem \eqref{eq:main_info} via
 \vspace{-0.2cm}
\begin{subequations}
\label{eq:prob_2}
\begin{align}
    \min \;\; & \eqref{eq:main_a_info}\\ \label{eq:relaxed_cons}
    \text{s.t.} \;\;\;& Q_{k}^{-1} \succeq (Q_{k-1}+S_{k-1})^{-1}+W, \; \forall k \in [K],\\
    &  \eqref{eq:main_collision_info} \; \text{and} \;\eqref{eq:main_final_info}.
    \end{align}
\end{subequations}

The following lemma proves that the optimal solution of Problem \eqref{eq:prob_2}  is also an optimal solution of Problem \eqref{eq:main_info}. Solving problem \eqref{eq:prob_2} has computational advantage over solving problem \eqref{eq:main_info}, as constraint  \eqref{eq:relaxed_cons} is convex whereas \eqref{eq:main_lossless_info} is not. More precisely, \eqref{eq:relaxed_cons} can be written as a linear matrix inequality (LMI)    
\vspace{-0.2cm}
\begin{align}
\label{eq:kf}
\begin{bmatrix}
        Q_{k} & Q_{k} & Q_{k}W\\ 
        Q_{k} & Q_{k-1}+S_{k-1} & 0\\
        W Q_{k} & 0 & W
\end{bmatrix} \succeq 0, \forall k \in [K].
\end{align}
\vspace{-0.2cm}
\begin{lemma}
\label{lemma_relax}
\textup{Let $(x^*_k, Q^*_k, S^*_k)$ be 
an optimizer for \eqref{eq:prob_2}. Then,
\vspace{-0.2cm}
\begin{align}\label{eq:relax}
Q_{k}^{*-1} = ({Q}_{k-1}^{*}+S^*_{k-1})^{-1}+W,\quad \forall k\in [K].
\end{align}
starting from ${Q}^{*}_0+S^{*}_0 =P_{0}^{-1}$.}
\end{lemma}

\begin{proof}
The proof is based on contradiction. Assume the optimal value of Problem \eqref{eq:prob_2} is  attained by $(x^*_k, {Q}^*_k, S^*_k)$ as $J^*$,  where \eqref{eq:relax} does not hold. We consider the set  $(x^*_k, {Q}^{**}_k, S^*_k)$, where
\vspace{-0.1cm}
\begin{align*}
{Q}_{k}^{**-1} = ({Q}_{k-1}^{**}+S^*_{k-1})^{-1}+W,\quad \forall k\in [K],
\end{align*}
starting again from ${Q}_0^{**}+S^*_{0}=P_0^{-1}$,
and show  $(x^*_k, {Q}^{**}_k, S^*_k)$ is a feasible solution to Problem (\ref{eq:prob_2}) that attains a lower value $J^{**}\leq J^{*}$.

\begin{claim}
\textup{For $k \in [K]$, we have}
\vspace{-0.2cm}
\begin{align}\label{eq:recur}
    {Q}_{k}^{**}& \succeq {Q}_{k}^{*}.
\end{align}
\end{claim}
\begin{proof}
We proceed via induction on $k$. For $k=1$, we have ${Q}_{1}^{**-1} = {Q}_{1}^{*-1} = {P}_{0} + W$, and thus the base step 
holds. We now assume \eqref{eq:recur} holds for $k=t$ and show \eqref{eq:recur} holds for $k=t+1$. From ${Q}_{t}^{**} \succeq {Q}_{t}^{*}$, it immediately follows that
${Q}_{t+1}^{**-1} = ({Q}_{t}^{**}+S^*_t)^{-1} + W \preceq  ({Q}_{t}^{*}+S_t^*)^{-1} + W \preceq  {Q}_{t+1}^{*-1}$,
which establishes the claim for $k=t+1$.
\end{proof}

Claim~1 implies that $\mathcal{E}_{\chi}(x^*_{k}[s], ({Q}_{k}^{**}+S_k^{*})^{-1}[s]) \subseteq \mathcal{E}_{\chi}(x^*_{k}[s], ({Q}_{k}^{*}+S^*_k)^{-1}[s])$ and  $\mathcal{E}_{\chi}(x^*_{K},(\hat{Q}_{K}^{**}+S^*_K)^{-1}) \subseteq \mathcal{E}_{\chi}(x^*_{K},(\hat{Q}_{K}^{*}+S^*_K)^{-1}), $ which proves $(x^*_k, {Q}^{**}_k, S^*_k)$ satisfies constraints \eqref{eq:main_collision_info}, and \eqref{eq:main_final_info}. Constraint \eqref{eq:main_lossless_info} is also satisfied trivially which leads to the conclusion that $(x^*_k, {Q}^{**}_k, S^*_k)$ is a feasible solution for \eqref{eq:main_info}. On the other hand, using the matrix determinant lemma we have
\vspace{-0.2cm}
\begin{align} \nonumber
    &\log\det (Q_k+S_{k}) - \log\det Q_{k}\\ \label{eq:obj_decrease}
    &= \log\det S_k+ \log\det({Q}_k^{-1}+S_k^{-1}).
\end{align}
It is trivial to see \eqref{eq:obj_decrease} is decreasing function of ${Q}_k$, which proves $J^{**}\leq J^{*}$.
\end{proof}

Both constraints  \eqref{eq:main_collision_info}, and \eqref{eq:main_final_info} have to be held over a continuous domain (like $\forall x_{\text{o}}\in\mathcal{X}_{\text{O}}^j$), which cannot be handled directly by standard gradient-based solvers. In the upcoming subsections, we derive equivalent conditions for these constraints using strong duality and the theorem of alternatives \cite{boyd2004convex}. 
\vspace{-0.2cm}
\subsection{Discrete-time Collision Constraint}
\begin{lemma}\label{lemma:overlap}
The ellipse $\mathcal{E}_{\chi}(x_i,Q_i^{-1})$ and the $f$-faced polyhedron $\mathcal{X}=\{ x: A x\leq b\}$  do not overlap if and only if $\exists \lambda \geq 0_{f}$ such that
\begin{align}
\label{eq:non_col}
-\lambda^\top A Q_i^{-1} A^\top \lambda + 2 \lambda^\top (A x_i- b) \geq \chi^2.
\end{align}
\end{lemma}

\begin{proof}
Invoking the definition of confidence ellipse, the absence of overlap between $\mathcal{E}_{\chi}(x_i,Q_i^{-1})$ and $\mathcal{X}$ can be written as 
\vspace{-0.2cm}
\begin{equation}
\nonumber
    \forall x \in \mathcal{X}, \quad  (x-x_i)^\top Q_i (x-x_i) \geq \chi^2,
\end{equation}
which is equivalent to the condition that the minimum value 
\begin{equation}
\label{eq:non_col_2}
     V^* \triangleq \min_{x:Ax\leq b}  (1/2)(x-x_i)^\top Q_i (x-x_i)
\end{equation}
is greater than $ \chi^2/2$. It is straightforward to see that the dual of \eqref{eq:non_col_2} is 
\vspace{-0.2cm}
\begin{equation}
    \max_{\lambda\geq 0} -(1/2) \lambda^\top A Q_i^{-1} A^\top \lambda+ \lambda ^\top (A x_i - b),
\end{equation}
where $\lambda \geq 0_f$ is the dual variable. The optimization problem \eqref{eq:non_col_2} is convex in $x$ for a given pair of $(x_i,P_i)$, and it is easy to verify that Slater's condition holds. Therefore, strong duality holds, implying that $V^* \geq \chi^2/2$ if and only if there exists a dual feasible solution ($\lambda \geq 0_f$) satisfying \eqref{eq:non_col}.  
\end{proof}

\begin{remark}
\label{remark:one}
If the polyhedron $\mathcal{X}$ is a half-space $\{x: a^\top x \leq b\}$,
 \eqref{eq:non_col} simplifies to $- (a^\top Q_i^{-1} a) \lambda^2 + 2(a^\top x_i- b) \lambda -\chi^2 \geq 0$,
which is a second order function of scalar $\lambda$. It is easy to verify that the maximum of $\frac{(a^\top x_i-b)^2}{a^\top Q_i^{-1} a}-\chi^2$ is obtained at $\lambda=\frac{a^\top x_i-b}{a^\top Q^{-1}a} \geq 0$. Hence, \eqref{eq:non_col} reduces to 
\vspace{-0.2cm}
\begin{align}
\label{eq:non_col2}
     a^\top x_i+ \sqrt{\chi^2 a^\top  Q_i^{-1} a} \leq b.
\end{align}
\end{remark}
Equation \eqref{eq:non_col2} was previously derived in \cite{van2002closed} and is extensively used in CC planners. Constraint \eqref{eq:non_col2} is not convex in $(x_i,Q^{-1}_i)$. In the following lemma, we derive an equivalent convex condition.  
\begin{lemma}
\label{eq:lemma:C_intro}
The relation \eqref{eq:non_col2} holds if and only if there exists a $ C\geq 0 $ such that
\begin{align}
\label{eq:lemma_3}
    \begin{bmatrix}
           b-a^\top x_i & 1\\
           1 & \gamma C 
    \end{bmatrix} \succeq 0, \quad \text{and} \quad 
    \begin{bmatrix}
           1 & Ca^\top\\
           aC &  Q_{i} 
    \end{bmatrix} \succeq 0,
\end{align}
where $\gamma= (\chi^2)^{-\frac{1}{2}}$.
\end{lemma}
\begin{proof}
It is easy to verify that \eqref{eq:non_col2} holds if and only if  there exists a $ C \geq 0$ such that $a^\top Q_{i}^{-1}a \leq C^{-2}$ and $a^\top x_i + \gamma^{-1} C^{-1}< b$. By applying Schur complement lemma to these inequalities, relation \eqref{eq:lemma_3} is obtained.
\end{proof}

Using Lemma~\ref{eq:lemma:C_intro}, we can rewrite \eqref{eq:main_final_info} as a set of convex constraints 
\begin{subequations}
\label{eq:final_cond}
    \begin{align}
        &\begin{bmatrix}
           -b_{\text{tar},n}+a_{\text{tar},n}^\top x_K & 1\\
           1 & \gamma C_{n} 
    \end{bmatrix} \succeq 0, \; \forall n \in [N],\\
    &\begin{bmatrix}
           1 & -C_{n}a_{\text{tar},n}^\top\\
           -a_{\text{tar},n}C_{n} &  {Q}_{K}+S_{K} 
    \end{bmatrix} \succeq 0, \; C_n\geq 0, \;   \forall n \in [N].      
    \end{align}
\end{subequations}
\subsection{Continuous-time Collision Constraint}
Based on Definition~\ref{def:three}, we say a collision in transition $t_{k-1} \rightarrow t_k$ with polyhedral obstacle $\mathcal{X} =\{x: Ax\leq b\}$ is detected when
\vspace{-0.2cm}
\begin{equation*}
(x_k[s]-x)^\top P_k[s]^{-1}(x_k[s]-x) < \chi^2,
\end{equation*}
for some $ s \in int$ and $x\in \mathcal{X}$. Collision detection can be formulated as the feasibility problem w.r.t $s$ and $x$:
\begin{align}
&\begin{bmatrix}
\chi^2 &  x_{k-1}^\top+ s \Delta x^\top_{k}-x^\top \\ \nonumber
x_{k-1}+ s \Delta x_{k}-x & P_{k-1}+s W
\end{bmatrix}\succ 0,\\ \label{eq:feas_prob}
&0\leq s \leq 1,  \quad A x \leq b, 
\end{align}
which is a convex program  for arbitrary polyhedron $\mathcal{X}$. More precisely,  transition $t_{k-1} \rightarrow t_{k}$ is not in collision with $\mathcal{X}$ iff \eqref{eq:feas_prob} is infeasible. 
\begin{theorem}
\label{theo:one}
Problem \eqref{eq:feas_prob} is infeasible for $f$-faced polyhedral obstacle $\mathcal{X} =\{x: A x\leq b\}$  iff there exists a $\lambda\geq 0_f$ such that
\vspace{-0.2cm}
\begin{subequations}
\label{eq:tran_col}
\begin{align} \label{eq:tran_col_1}
    &-\lambda^\top  A P_{k-1} A^\top \lambda + 2\lambda^\top(A x_{k-1}-b)  \geq \chi^2,\\ \label{eq:tran_col_2}
    & -\lambda^\top  A (P_{k-1}+W) A^\top + 2\lambda^\top(A x_{k}-b) \geq \chi^2.
\end{align}
\end{subequations}
\end{theorem}
\begin{proof}
Based on the theorem of alternatives; see Appendix~\ref{sec:AppenI} for details.
\end{proof}

The conditions \eqref{eq:tran_col_1} and \eqref{eq:tran_col_2} are similar to \eqref{eq:non_col}.
They imply that the neither the initial ellipse $\mathcal{E}(x_{k-1},P_{k-1})$ (when $s=0$) nor the final ellipse $\mathcal{E}(x_{k},\hat{P}_{k})$ (when $s=1$)  overlap with $\mathcal{X}$. However, 
\eqref{eq:tran_col} is stronger than the condition implying that initial and final confidence ellipses in transition $t_{k-1} \rightarrow t_{k}$ are separately collision-free because  \eqref{eq:tran_col_1} and \eqref{eq:tran_col_2} should be satisfied for a common $\lambda$.

Constraint \eqref{eq:tran_col_2} has a general non-convex form. Nevertheless, by writing \eqref{eq:tran_col_2} in terms of information matrices and introducing a new slack variable $R\geq 0$, we can reformulate it as an LMI and a DOC constraint as
\begin{subequations}
\label{eq:doc}
\begin{align}
    &h_1(R,\lambda,Q_{k})\!\triangleq \!\!\begin{bmatrix}
    R &  \lambda ^\top A\\
    A^\top\lambda & {Q}_{k}
    \end{bmatrix}\!\succeq 0, 
     \\ \nonumber &h_2(R,x_{k},\lambda)\triangleq\!R\!+\!\!\chi^2\!+\!\|A x_{k}-b\|^2\!+\!\lambda^2\!
     \\
     &\qquad \qquad \qquad -\!||Ax_{k}\!\!-\!b+\lambda||^2\!\leq 0.
\end{align}
\end{subequations}
An analogous reformulation can be performed for \eqref{eq:tran_col_1}.

\section{Algorithm}
\label{sec:Algo}
Substituting  \eqref{eq:final_cond} and \eqref{eq:doc}, Problem \eqref{eq:prob_2} becomes 
\vspace{-0.2cm}
\begin{subequations}
\label{eq:main_doc}
\begin{align}
\label{eq:main_a_doc}
    \min \;\; & \eqref{eq:main_a_info}\\
    \label{eq:main_kf_tar_doc}
   \text{ s.t.} \;\;&  \eqref{eq:kf} \; \text{and} \; \eqref{eq:final_cond},
   \\ \label{eq:main_c_doc}
   &\begin{bmatrix} 
    2 \lambda_{1,j}^\top (A_jx_0-b_j)-\chi^2 &  \lambda_{1,j}^\top A_j\\
    A_j^\top\lambda_{1,j} & {Q}_{0}+S_{0}
    \end{bmatrix} \succeq 0,\\
    \label{eq:main_d_doc}
    &h_1(R_{k,j}, \lambda_{k,j},Q_{k-1}\!+S_{k-1}) \!\succeq 0, \; \forall k\!\in\![2\!:\!K],\\
    \label{eq:main_e_doc}
    &h_1(\hat{R}_{k,j}, \lambda_{k,j},Q_{k}) \succeq 0,\; \forall k\in[K],\\ 
   \label{eq:main_f_doc}
    & h_2(R_{k,j}, x_{k-1}, \lambda_{k,j}) \leq 0,\;
    \forall k\in[2\!:\!K], \\ 
    \label{eq:main_g_doc}
    &h_2(\hat{R}_{k,j}, x_{k}, \lambda_{k,j}) \leq 0,
    \;\forall k\in[K],  
\end{align}
\end{subequations}
with variables $\{x_k,Q_k\succeq 0 ,S_k \succeq 0\}_{k=1}^K$, $\{C_n\geq 0\}_{n=1}^{N}$, $\{\lambda_{k,j}\geq 0, \hat{R}_{k,j} \geq 0\}_{k=1, j=1}^{K,J}$, and $\{ R_{k,j} \geq 0\}_{k=2, j=1}^{K,J}$. In \eqref{eq:main_doc}, we assumed $\mathcal{X}_{\text{O}}^j$ is defined as $\{x\in \mathbb{R}^d:A_j x \leq b_j\}$, and constraints \eqref{eq:main_c_doc}-\eqref{eq:main_g_doc} are imposed for all $j\in[J]$.  

In \eqref{eq:main_doc}, all terms in the objective function except $\log\det(Q_k+S_k)$, and all terms in constraints except
$-||A_jx_{k-1}-b_j+\lambda_{k,j}||^2$ in \eqref{eq:main_f_doc} and $-||A_jx_k-b_j+\lambda_{k,j}||^2$ in \eqref{eq:main_g_doc} are convex. These non-convex terms are negative of convex functions, meaning that \eqref{eq:main_doc} is a DOC problem.

A variety of  sequential quadratic programming (SQP)-based
approaches \cite{boggs1995sequential} could be used to solve the nonlinear program \eqref{eq:main_doc} to local optimality. However, it would be required to  artificially assume that the sequence of convex programs in the SQP solvers stay feasible. In contrast, if we apply CCP to a DOC problem like \eqref{eq:main_doc},  the concavity of the non-convex terms guarantees that the sequence of convex programs is feasible \cite{lipp2016variations}. Also, it is shown that CCP monotonically converges to a local optimum \cite{lipp2016variations}. 

\subsection{Convex-Concave Procedure (CCP)}
CCP is an iterative method that starts from a feasible solution of a DOC optimization program. It over-approximates concave terms (both in the objective function and constraints) in the program via affine functions obtained by linearization around a feasible solution. The resulting convex problem can then be solved using standard convex solvers. The linearization is then repeated around the obtained solution, and the iteration continues until the sequence of solutions converges to a locally optimal solution \cite{lipp2016variations}.

To implement CCP for \eqref{eq:main_doc}, we can linearize the functions  $h_3(Q,S)\triangleq \log\det (Q+S)-\log\det Q$ and $h_2(R,x,\lambda)$ around  a feasible solution $(\tilde{Q},\tilde{S}, \tilde{x}, \tilde{\lambda})$ by $\bar{h}_3(Q,S: \tilde{Q},\tilde{S}) = \log\det(\tilde{Q}+\tilde{S})-\log\det Q+ \textup{Tr}\left((\tilde{Q}+\tilde{S})^{-1} (Q+S- \tilde{Q}-\tilde{S})\right)$, and $\bar{h}_2(R, x,\lambda: \tilde{x},\tilde{\lambda}) = R\!+\!\!\chi^2\!+\!\|A x_{k}-b\|^2\!+\!\lambda^2 -
\|A\tilde{x}-b+\tilde{\lambda}\|^2 - 2
    (A\tilde{x}-b+\tilde{\lambda})^\top (A(x-\tilde{x})+\lambda-\tilde{\lambda})$, respectively.
Denoting the the solution obtained via the  CCP iteration $i$ as $\{{x}_k^i, {S}_k^i, {Q}^i_{k}\}_{k=1}^{K}$ and $\{\lambda_{k,j}^i\}_{k=1,j=1}^{K,J}$, at iteration $i+1$ we solve the convex problem
\begin{subequations}
\vspace{-0.4cm}
\begin{align}
\nonumber
    \min \;\; & \sum_{k=1}^{K} \|x_{k}-x_{k-1}\|^2 \! +\frac{\alpha}{2} \bar{h}_3(Q_k,S_k:Q^i_k,S^i_k)\\
    \nonumber
   \text{ s.t.} \;\;&  \eqref{eq:main_kf_tar_doc}- \eqref{eq:main_e_doc},
    \\\nonumber
    &\bar{h}_2(R_{k,j},x_{k-1},\lambda_{k,j}: x_{k-1}^i,\lambda_{k,j}^i) \leq 0, \\ \nonumber
    &\forall k\in[2\!:\!K], \forall j \in[J],\\
    \nonumber
    &\bar{h}_2(\hat{R}_{k,j}, x_{k},\lambda_{k,j}: x^i_{k}, \lambda_{k,j}^i)\leq 0, \\ \nonumber
    &\forall k\in[K], \quad \forall j \in[J].
\end{align}
\end{subequations}

\vspace{-0.4cm}
\subsection{Initialization}
The first stage of CCP starts from an initial feasible solution $\{x^0_k, S_k^0, Q^0_{k}\}_{k=1}^{K}$ and $\{\lambda_{k,j}^0\}_{k=1,j=1}^{K,J}$. However, the IG-RRT* algorithm does not explicitly provide a set of $\{\lambda_{k,j}^0\}_{k=1,j=1}^{K,J}$. In IG-RRT*, the collisions are checked through a numerical state-validator function (like any sampling-based method),  and not directly through \eqref{eq:tran_col}.
Nevertheless, a feasible $\lambda^0_{k,j}$ that satisfies \eqref{eq:tran_col}
for polyhedral obstacles $\mathcal{X}_{\text{o}}^{j}=\{x:A_j x\leq b_j\}$ and a feasible set $\{x^0_k, S^0_k, Q^0_{k}\}_{k=1}^K$ obtained from IG-RRT* can be sought by solving the convex feasibility problem  
\begin{align*}
\quad & \begin{bmatrix}
2 \lambda_{k,j}^\top (A_j x^0_{k-1}-b_j)-\chi^2 & \lambda_{k,j}^\top A_j\\
A_j^\top \lambda_{k,j} & Q_{k-1}^{0}+S^{0}_{k-1}
\end{bmatrix} \geq 0, \\
& \begin{bmatrix}
2 \lambda_{k,j}^\top (A_jx^0_k-b_j)-\chi^2 & \lambda^\top_{k,j} A_j\\
A_j^\top \lambda_{k,j} & Q^0_{k}
\end{bmatrix} \geq 0, 
\quad \lambda_{k,j} \geq 0, 
\end{align*}
w.r.t $\lambda_{k,j}$.

\vspace{-0.4cm}
\section{Simulation Results}\label{sec:Results}
Fig.~\ref{fig:initial} shows a $1\text{m}\times 1\text{m}$ obstacle-filled environment, and the path plans obtained by the IG-RRT* algorithm with $N=500$ and $W=0.2\times 10^{-3} I_2$ for two values of $\alpha=0.1$ and $\alpha=1.0$. Fig.~\ref{fig:final} depicts the smoothed versions of these paths obtained after $15$ CCP iterations. A comparison between Fig.~\ref{fig:initial} and Fig.~\ref{fig:final} reveals the success of the proposed algorithm in smoothing the path plans obtained from the stochastic IG-RRT algorithm.  
The videos showing the evolution of the paths during the smoothing algorithm for $\alpha=0.1$ and $1.0$ are accessible at \href{https://youtu.be/ieUbd1uj-aE}{https://youtu.be/ieUbd1uj-aE} and \href{https://youtu.be/nNZn4GGbbWs}{https://youtu.be/nNZn4GGbbWs}, respectively. Fig.~\ref{fig:cost} demonstrates the monotonic reduction of steering costs in the sequence of path plans obtained in CCP iterations.  
\begin{figure}[ht!]
\vspace{-0.5cm}
    \centering
    \subfloat[$\alpha=0.1$.]
    {\includegraphics[trim = 2.1cm 0cm 0.0cm 0.4cm, clip=true, width=0.55\columnwidth]{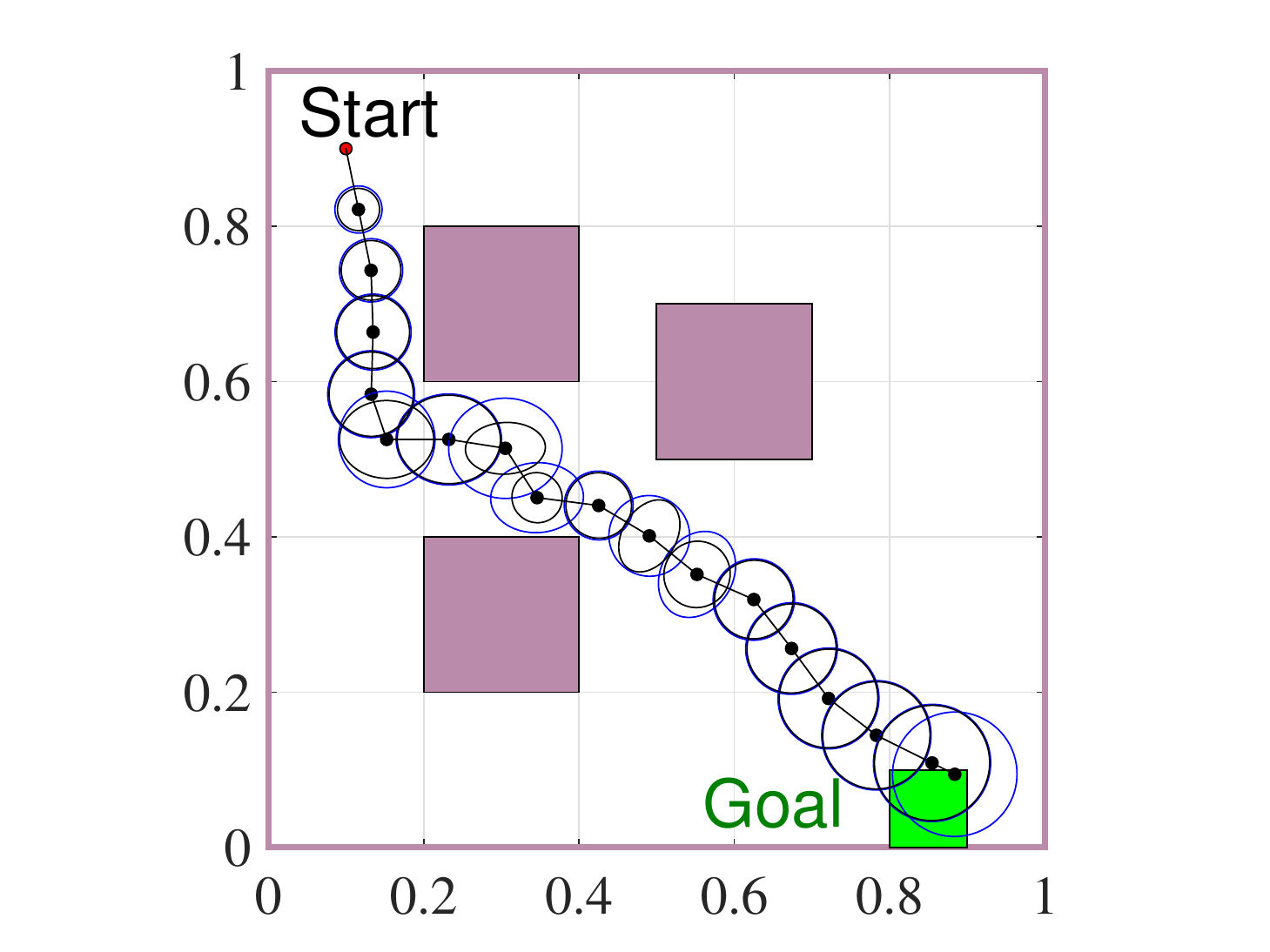}
    \label{fig:init_01}} 
    \subfloat[$\alpha=1.0$.] 
    {\includegraphics[trim =2.1cm 0cm 0.0cm 0.4cm, clip=true, width=0.55\columnwidth]{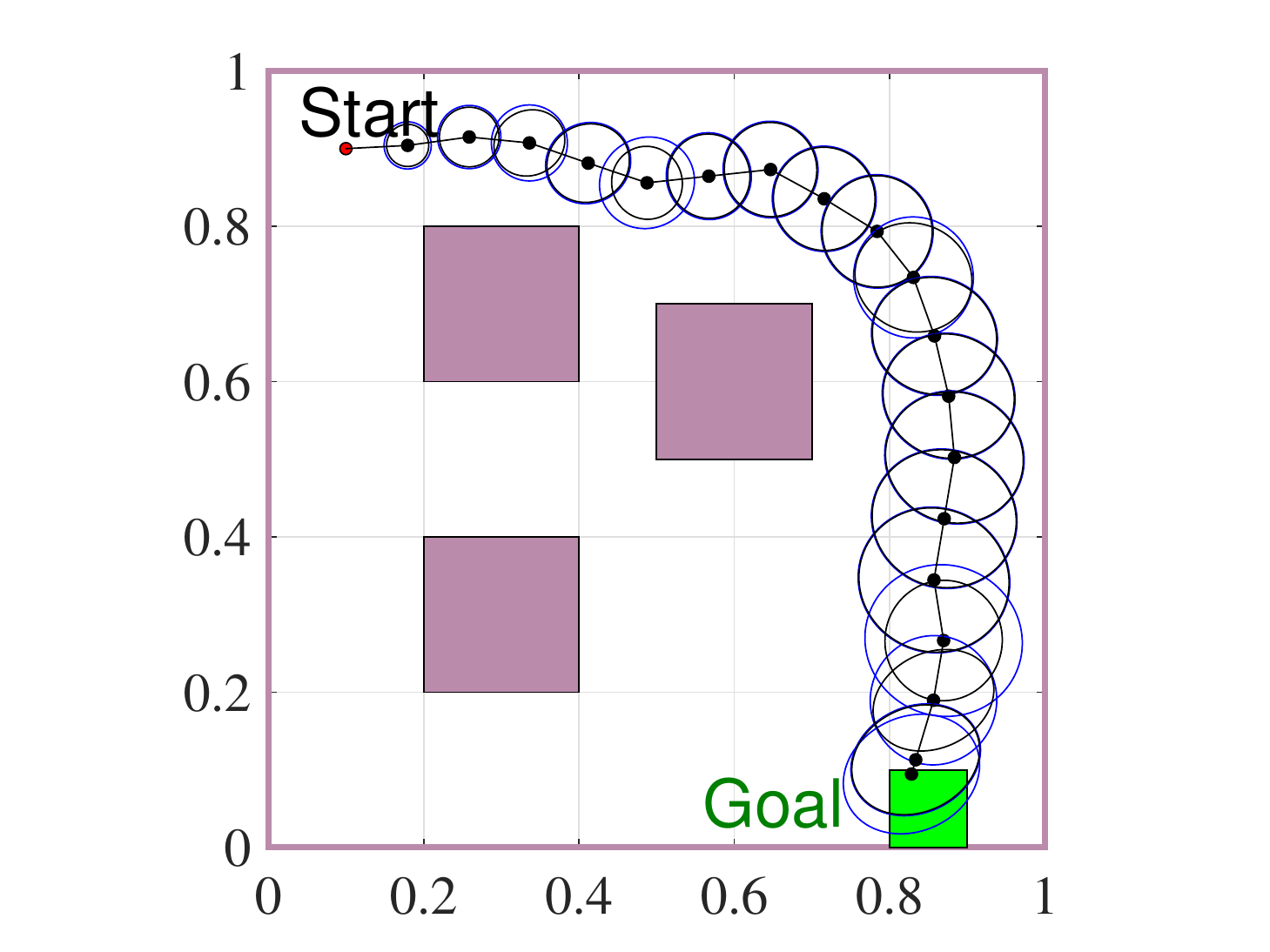}
    \label{fig:init_1}}
    \caption{Path plans obtained from IG-RRT* algorithms with $500$ nodes with $W=0.2\times 10^{-3} I_2$ m$^2$ in a $1\text{m}\times 1\text{m}$ obstacle-filled environment. The blue ellipses show $90\%$  prior confidence ellipses while the black ellipses show $90\%$  posterior confidence ellipses.}
    \label{fig:initial}
\end{figure}
\begin{figure}[ht!]
    \vspace{-0.5cm}
    \centering
    \subfloat[$\alpha=0.1$.]
    {\includegraphics[trim = 2.1cm 0cm 0.0cm 0.4cm, clip=true, width=0.55\columnwidth]{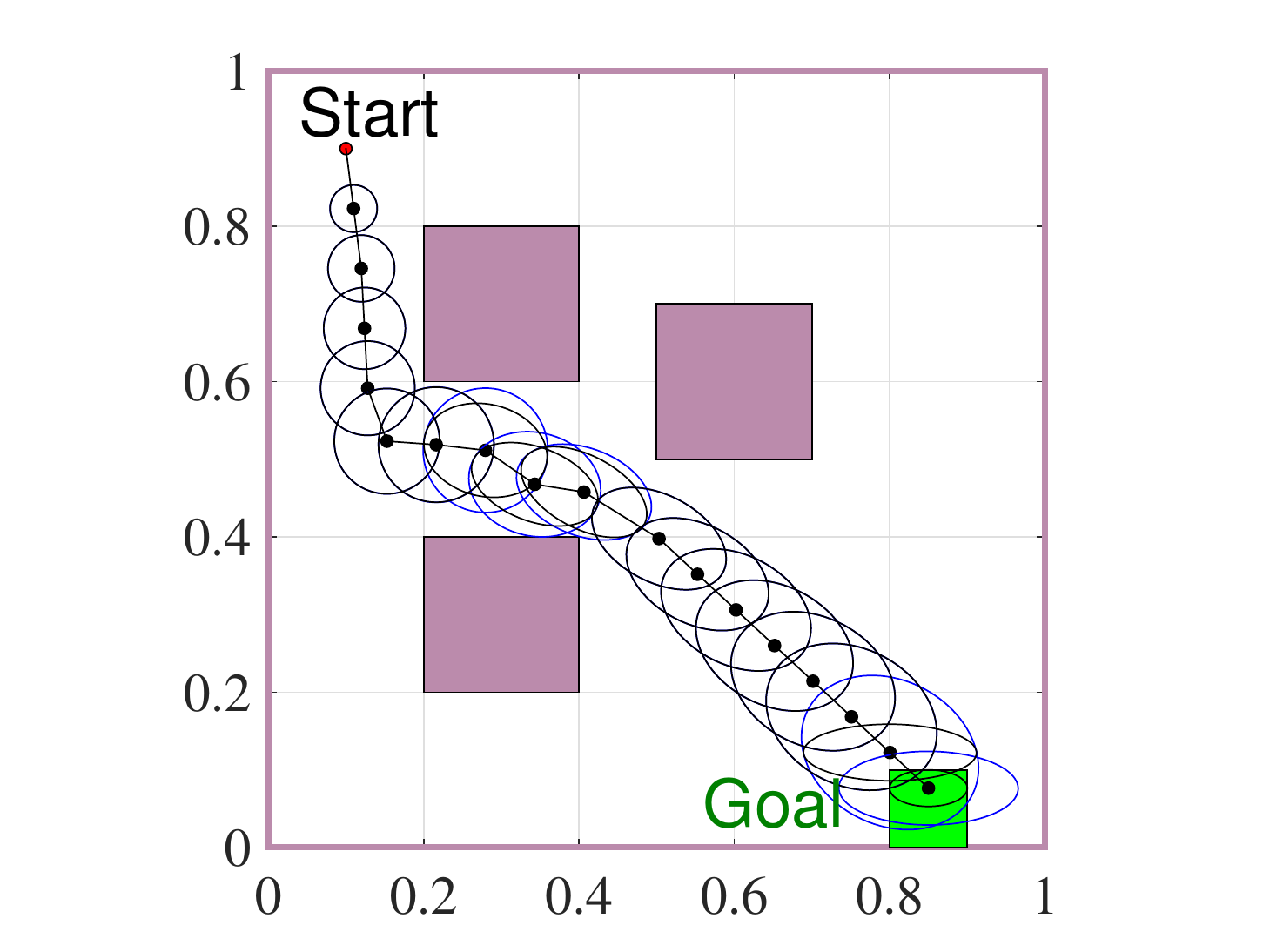}
    \label{fig:fin_01}} 
    \subfloat[$\alpha=1.0$.] 
    {\includegraphics[trim = 2.1cm 0cm 0.0cm 0.4cm, clip=true, width=0.55\columnwidth]{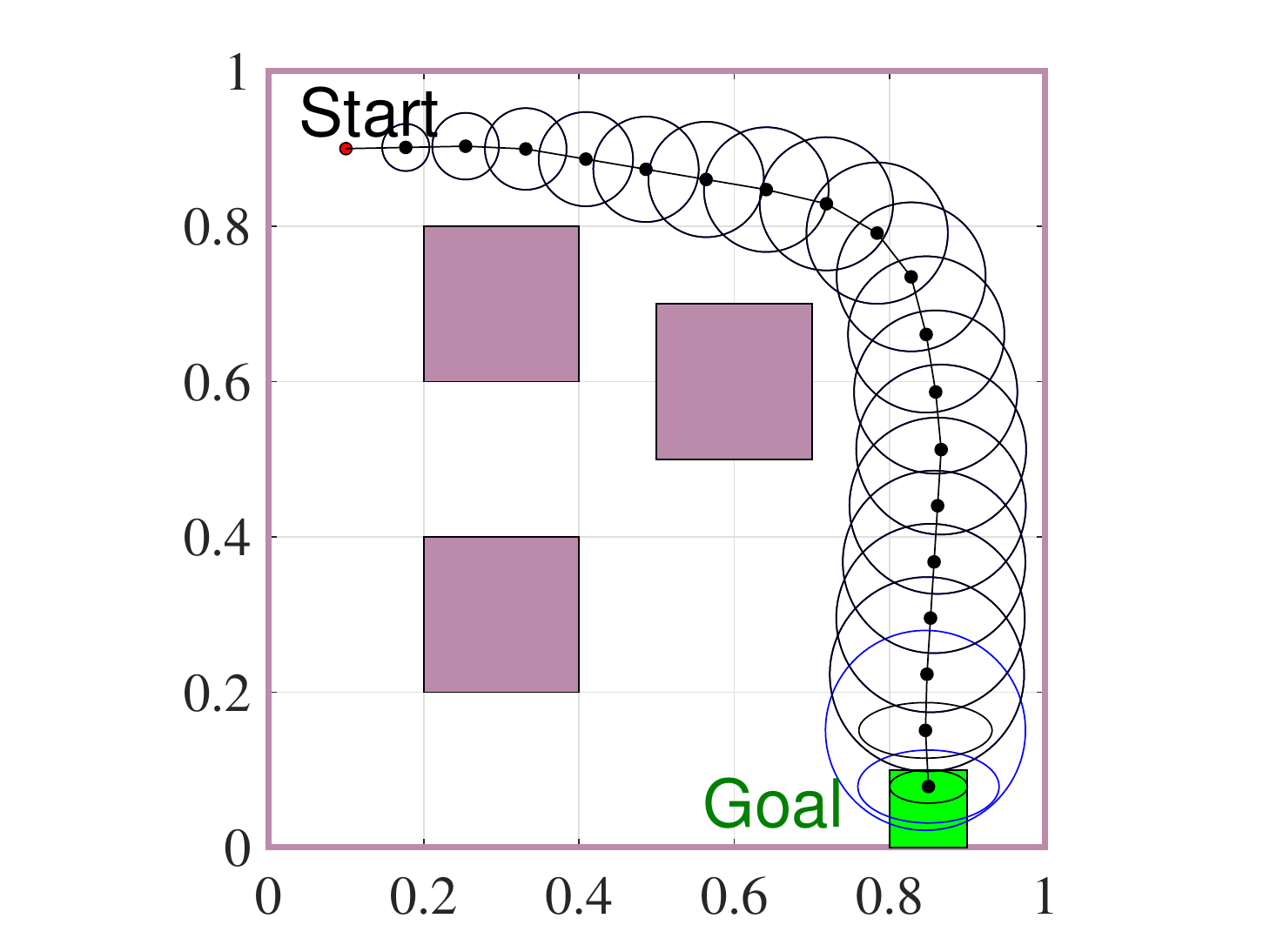}
    \label{fig:fin_1}}
    \caption{Path plans obtained after performing the proposed smoothing algorithm for the paths depicted in Fig.~\ref{fig:initial}. The blue ellipses show $90\%$  prior confidence ellipses while the black ellipses show posterior confidence ellipses.}
    \label{fig:final}
\end{figure}
\begin{figure}[ht!]
\vspace{-0.7cm}
\centering
\includegraphics[trim = 0cm 0cm 0cm 0cm, clip=true, width = 0.75\columnwidth]{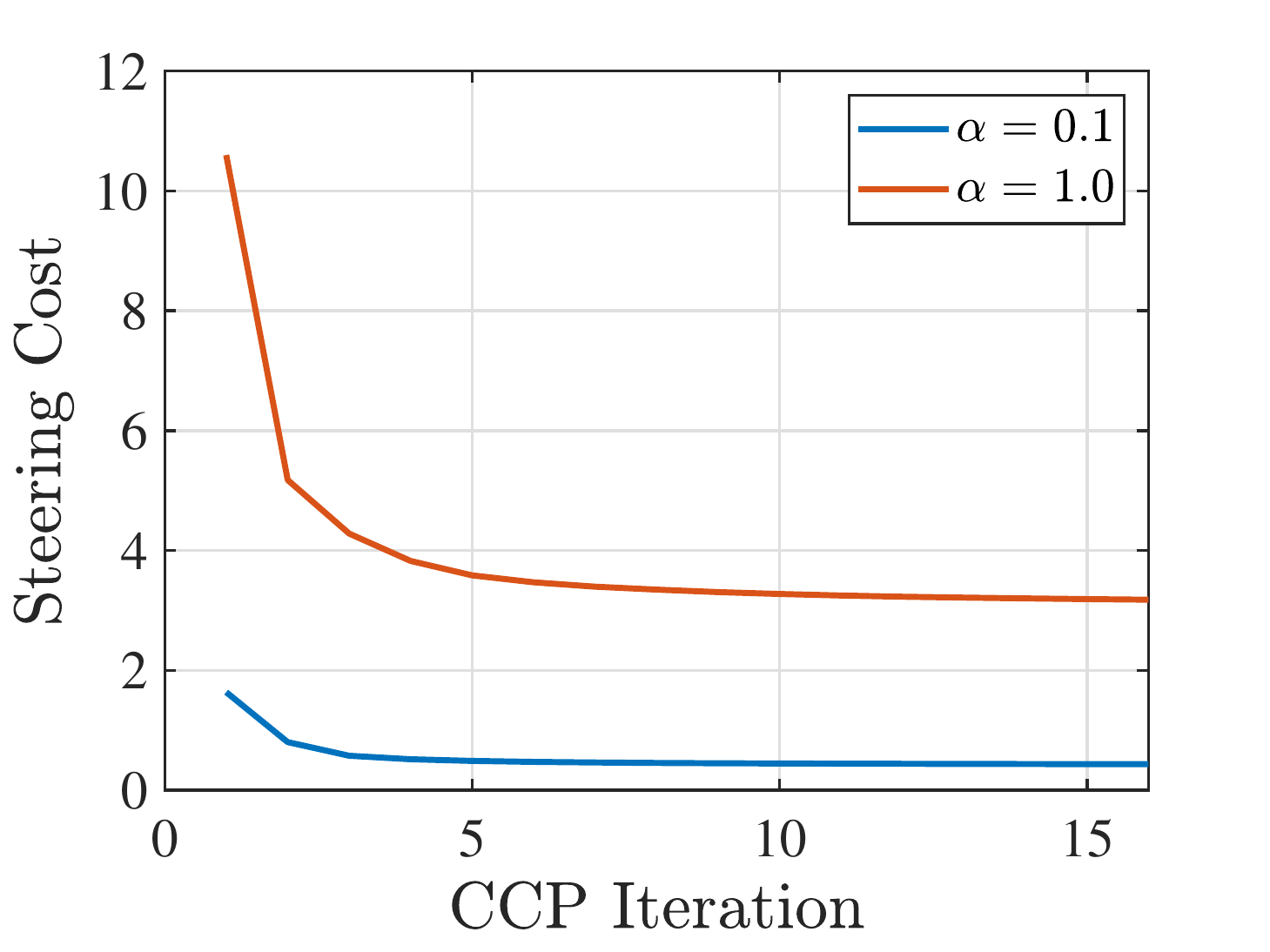}
\vspace{-0.2cm}
\caption{Steering cost versus CCP iteration.}
\label{fig:cost}
\end{figure}
\section{Conclusion and Future Work}
\label{sec:Conclusion}
In this work, we studied the smoothing of minimum sensing belief paths obtained by the IG-RRT* algorithm. We derived a novel safety constraint to bound the probability of collision with polyhedral obstacles in the transition between two Gaussian belief states. We deployed the presented safety constraint to formulate minimum sensing path planning as an optimization problem. We formulated this problem as a DOC program, for which the CCP algorithm can be utilized to find local optima. We proposed to use such a CCP algorithm as an efficient smoothing algorithm.
Numerical simulations demonstrated the utility of the proposed algorithm.

This paper assumes a fixed (common) confidence level for all
transitions on the path, and it is silent about the end-to-end
probability of collision. One future direction is to develop a
smoothing mechanism that allows the allocation of (potentially) different confidence
levels to each transition subject to a constraint on the end-to-end collision probability.
\begin{appendices}
\vspace{-0.2cm}
\section{Proof of Theorem~\ref{theo:one} } \label{sec:AppenI}
We provide the proof for $d=2$; however, the proof can be generalized for arbitrary dimension. 
If we denote the elements of $x$ as $x= [q_1\; q_2]^\top$, feasibility problem \eqref{eq:feas_prob} becomes
\vspace{-0.2cm}
\begin{subequations}
\label{eq:col_new}
\begin{align}
\min_{s\geq 0,\; q_1, q_2}\quad  & 0\\
\text{s.t.} \quad & 
R_0 +R_1 q_1+R_2 q_2+ R_3s\succ 0,\\
& s \leq 1,  \quad A_1 q_1 + A_2 q_2 \leq b, 
\end{align}
\end{subequations}
where
\vspace{-0.2cm}
\begin{align*}
    &R_0=\begin{bmatrix}
    \chi^2 & x_{k-1}^\top\\
    x_{k-1} & P_{k-1}
    \end{bmatrix}, \quad 
    R_1=\begin{bmatrix}
    0 & -1& 0\\
    -1 & 0&0\\
    0 & 0 & 0
    \end{bmatrix},\\
    & R_2=\begin{bmatrix}
    0 & 0& -1\\
    0 & 0&0\\
    -1 & 0 & 0
    \end{bmatrix}, 
    \quad R_3=\begin{bmatrix}
    0 & \Delta x_k^\top\\
    \Delta x_k & W
    \end{bmatrix},
\end{align*}
and $A_1$ and $A_2$ are the first and the second columns of $A$, respectively. The
Lagrangian of problem \eqref{eq:col_new}, can be written as $ \mathcal{L}(q_1,q_2,s, M, \gamma, \lambda) = \textup{Tr}\left(-(R_0 +R_1 q_1+R_2 q_2+ R_3 s) M\right)+ \gamma (s-1) +\lambda^\top(A_1 q_1 + A_2 q_2- b)$, where $M\succeq0$, $\gamma\geq0$, and $\lambda \geq 0$ are dual variables. Dual function for \eqref{eq:col_new} can be written as
$g(M, \gamma, \lambda) = -\textup{Tr}(R_0M)-\lambda^\top b-\gamma+ \inf_{q_1} (-\textup{Tr}(R_1M)+\lambda^\top A_1)   q_1 + \inf_{q_2} (-\textup{Tr}(R_2M)+\lambda^\top A_2)  q_2+ \inf_{s\geq0}(-\textup{Tr}(R_3M)+\gamma) s$. 
Hence, the dual problem of \eqref{eq:col_new} can be written as
\vspace{-0.2cm}
\begin{subequations}
\label{eq:dual}
\begin{align}\label{eq:dual_a}
    \max \quad & -\textup{Tr}(R_0M)-\lambda^\top b-\gamma\\ \label{eq:dual_b}
    \text{s.t.} \quad &  -\textup{Tr}(R_1M)+\lambda^\top A_1=0,\\\label{eq:dual_c}
    &  -\textup{Tr}(R_2M)+\lambda^\top A_2=0,\\ \label{eq:dual_d}
    &-\textup{Tr}(R_3M)+\gamma\geq 0.
\end{align}
\end{subequations}
 where variables are $M\succeq 0$, $\gamma\geq 0$, and $\lambda \geq 0$. If we denote the element at the $i$th row and the $j$th column of $M$ by $m_{i,j}$, 
constraints \eqref{eq:dual_b} and \eqref{eq:dual_c} yield $[m_{1,2}\; m_{1,3}] = -\frac{1}{2} \lambda^\top A$. Hence, $M$ has the form  
\vspace{-0.2cm}
\[M=\frac{1}{2}\begin{bmatrix}
M_1& - \lambda^\top A\\
-A^\top \lambda & M_2
\end{bmatrix} \succeq 0.\] 
Using the structure of $M$, Problem \eqref{eq:dual} turns to 
\vspace{-0.2cm}
\begin{align} \nonumber
    \max \quad  &-\frac{1}{2}M_1\chi^2+\!\lambda^\top(A x_{k-1}-b)\! -\frac{1}{2}\textup{Tr}(P_{k-1} M_2)\!-\!\gamma\\ \nonumber
    \text{s.t.} \quad & \begin{bmatrix}
M_1 & \lambda^\top A\\ \nonumber
A^\top \lambda & M_2
\end{bmatrix} \succeq 0, 
\frac{1}{2}\textup{Tr}(WM_2)- \lambda^\top A \Delta x_k  \leq \gamma,
\end{align}
where the variables are $M_1>0, M_2\succeq 0, \lambda\geq 0,$ and $\gamma \geq 0$.
After performing the maximization w.r.t $\gamma$, 
this optimization simplifies to 
\begin{subequations}
\label{eq:dual_3}
\begin{align}
\nonumber
    \max_{M_1\geq 0, M_2\succeq 0, \lambda\geq 0}\!\! &-\!\frac{1}{2}M_1\chi^2\!+\!\lambda^\top\!(A x_{k-1}\!-b)\!-\!\frac{1}{2}\textup{Tr}(P_{k-1} M_2\!)\\
    &+\!\min\{0,-\frac{1}{2}\textup{Tr}(WM_2)\!+\!\lambda^\top\!A \Delta x_k\}\\ \label{eq:dual_3_b}
    \text{s.t.} \quad \quad & \begin{bmatrix}
M_1 & \lambda^\top A\\
A^\top \lambda & M_2
\end{bmatrix} \succeq 0.
\end{align}
\end{subequations}
The objective function and the constraints in \eqref{eq:dual_3} are affine in dual variables $M_1$, $M_2$, and $\lambda$. Thus, \eqref{eq:dual_3} is unbounded iff it admits a feasible solution $(M_1>0,  M_2\succeq 0, \lambda \geq0,)$ that yields a positive value of objective function. From the theorem of alternatives, we know unboundedness of dual problem implies the infeasibility of primal problem and vice versa. Hence, \eqref{eq:col_new} is infeasible iff $\exists \lambda \geq0, M_2\succeq 0$ such that
\vspace{-0.6cm}
\begin{subequations}
\label{eq:cond}
\begin{align}
\label{eq:cond_a}
  &2\lambda^\top(A x_{k-1}-b) - \textup{Tr}(P_{k-1} M_2) \geq \chi^2,\\ \label{eq:cond_b}
   &2\lambda^\top(A x_{k}-b) - \textup{Tr}((P_{k-1}+W) M_2) \geq \chi^2,\\ \label{eq:cond_c}
  & M_2\succeq \lambda^\top A A^\top \lambda, 
\end{align}
\end{subequations}
where w.l.o.g we assumed $M_1=1$ and \eqref{eq:cond_c} is obtained by applying Schur complement lemma to \eqref{eq:dual_3_b}.
It is easy to verify that the maximum values of the LHS of both \eqref{eq:cond_a} and \eqref{eq:cond_b} subject to \eqref{eq:cond_c} are obtained at $M_2=\lambda^\top A A^\top \lambda$. Hence, condition \eqref{eq:cond} can be equivalently written as $\exists \lambda\geq0$ such that \eqref{eq:tran_col} holds, which completes the proof.

\end{appendices}

\addtolength{\textheight}{-0cm}   

\bibliographystyle{IEEEtran}
\bibliography{ref_shorten.bib}

\end{document}